%% file: linear_net_overparam.tex
\documentclass{article}

\usepackage{microtype}
\usepackage{graphicx}
\usepackage{subfigure}
\usepackage{booktabs} 

\usepackage{hyperref}

\usepackage[accepted]{icml2019}

\usepackage{natbib}
\usepackage{amsmath}
\usepackage{amsthm}
\usepackage{amssymb}
\usepackage{mathtools}
\usepackage{tikz}
\usepackage{xcolor}
\usetikzlibrary{arrows}
\usepackage{dsfont}
\usepackage{comment}
\usepackage{enumerate}

\newcommand{\argmin}{\mathrm{argmin}}

\newcommand{\poly}{\mathrm{poly}}
\newcommand{\rank}{\mathrm{rank}}

\newcommand{\tr}{\mathrm{tr}}
\newcommand{\opt}{\mathsf{OPT}}

\newcommand{\supp}{\mathrm{supp}}

\def\R{\mathbb{R}}

\def\cA{\mathcal{A}}
\def\cB{\mathcal{B}}
\def\cC{\mathcal{C}}

\def\cN{\mathcal{N}}

\def\cS{\mathcal{S}}

\def\din{d_{\mathrm{in}}}
\def\dout{d_{\mathrm{out}}}

\newcommand{\norm}[1]{\left\|#1\right\|}
\newcommand{\abs}[1]{\left|#1\right|}
\newcommand{\expect}[1]{\mathbb{E}\left[#1\right]}

\newcommand{\vectorize}[1]{\text{vec}\left(#1\right)}

\usepackage{bm}

\newcommand{\eps}{\varepsilon}

\newtheorem{thm}{Theorem}[section]

\newtheorem{lem}[thm]{Lemma}

\newtheorem{prop}[thm]{Proposition}

\newtheorem{claim}[thm]{Claim}

\newtheorem{asmp}{Assumption}[section]

\allowdisplaybreaks

\icmltitlerunning{Width Provably Matters in Optimization for Deep Linear Neural Networks}

\begin{document}

\twocolumn[
\icmltitle{Width Provably Matters in Optimization for Deep Linear Neural Networks}



\icmlsetsymbol{equal}{*}

\begin{icmlauthorlist}
\icmlauthor{Simon S. Du}{equal,cmu}
\icmlauthor{Wei Hu}{equal,pu}

\end{icmlauthorlist}

\icmlaffiliation{cmu}{Carnegie Mellon University, Pittsburgh, PA, USA}
\icmlaffiliation{pu}{Princeton University, Princeton, NJ, USA}
\icmlcorrespondingauthor{Simon S. Du}{ssdu@cs.cmu.edu}
\icmlcorrespondingauthor{Wei Hu}{huwei@cs.princeton.edu}

\icmlkeywords{neural networks, optimization}

\vskip 0.3in
]



\printAffiliationsAndNotice{\icmlEqualContribution} 

\begin{abstract}
	\input{abs.tex}

\end{abstract}

\section{Introduction}
\label{sec:intro}

\input{intro.tex}

\section{Related Work}
\label{sec:rel}

\input{rel.tex}

\section{Preliminaries}
\label{sec:pre}

\input{pre.tex}

\section{Main Result}
\label{sec:result}
\input{result.tex}

\section{Proof Overview}
\label{sec:proof-overview}

\input{proof-overview.tex}

\section{Properties at Initialization}
\label{sec:init}

\input{init.tex}

\section{Proof of the Main Theorem}
\label{sec:proof-main}

\input{proof_main.tex}

\section{Conclusion}
\label{sec:conclusion}

\input{conclusion.tex}

\section*{Acknowledgments}
\input{ack.tex}

\bibliography{refs}
\bibliographystyle{icml2019}

\onecolumn
\newpage
\appendix

\section*{\Large Appendix}

\input{appendix.tex}

\end{document}

%% file: abs.tex
We prove that for an $L$-layer fully-connected linear neural network, if the width of every hidden layer is $\widetilde{\Omega}\left(L \cdot r \cdot \dout \cdot \kappa^3 \right)$, where $r$ and $\kappa$ are the rank and the condition number of the input data, and $\dout$ is the output dimension, then gradient descent with Gaussian random initialization converges to a global minimum at a linear rate. The number of iterations to find an  $\epsilon$-suboptimal solution is $O(\kappa \log(\frac{1}{\epsilon}))$. Our polynomial upper bound on the total running time for wide deep linear networks and the  $\exp\left(\Omega\left(L\right)\right)$ lower bound for narrow deep linear neural networks [Shamir, 2018] together demonstrate that wide layers are necessary for optimizing deep models.

%% file: intro.tex
Recent success in machine learning involves training deep neural networks using randomly initialized first order methods, which requires optimizing highly non-convex functions.
Compared with nonlinear deep neural networks, deep linear networks are arguably more amenable to theoretical analysis.
It is widely believed that deep linear networks already captures important aspects of optimization in deep learning~\citep{saxe2014exact}.
Therefore,
 theoreticians have tried to study this problem in recent years.
However,  a strong global convergence guarantee is still missing.


%
A series of recent papers analyzed landscape of the deep linear network optimization problem~\citep{kawaguchi2016deep,hardt2016identity,lu2017depth,yun2017global,zhou2018critical,laurent2018deep}. However, these results do not imply convergence of gradient-based methods to the global minimum.
Recently, \citet{bartlett2018gradient, arora2018convergence} directly analyzed the trajectory generated by gradient descent, and showed that gradient descent converges to global minimum under further assumptions on both data and global minimum.
These results require specially designed initialization schemes, and do not apply to commonly used random initializations. In Section~\ref{subsec:compare} we describe above results in more details.


A recent work by \citet{shamir2018exponential} showed an exponential lower bound of randomly initialized gradient descent for narrow linear neural networks.
More precisely, he showed that for an $L$-layer linear neural network in which the input, output and all hidden dimensions are equal to $1$, gradient descent with Xavier initialization~\citep{glorot2010understanding} requires at least $\exp\left(\Omega\left(L\right)\right)$ iteration to converge.
This result demonstrates the intrinsic difficulty of optimizing deep networks: even in the basic setting, the convergence time for randomly initialized gradient descent can be exponential in depth.
Nevertheless, this lower bound only holds for narrow neural networks.
It is possible that making the hidden layers wider (which is usually the case in practice) can eliminate such exponential dependence on depth.
This gives rise to the following questions:
\begin{center}
\emph{Can randomly initialized gradient descent optimize \textbf{wide} deep linear networks in polynomial time? If so, what is a sufficient width in hidden layers?}
\end{center}

\paragraph{Our Contribution:}
We answer the first question positively and give a concrete quantitative result for the second question. 
We prove that as long as the width of hidden layers is at least $\tilde\Omega(L)$\footnote{We omit dependence on other parameters here. See Theorem~\ref{thm:main} for the precise requirements.}, gradient descent with Xavier  initialization with high probability converges to the global minimum of the $\ell_2$ loss at a linear rate \emph{under no assumption}.
To our knowledge, this is the first polynomial time global convergence guarantee for randomly initialized gradient descent for deep linear networks.
Furthermore, our convergence rate is tight in the sense that it matches the convergence rate of applying gradient descent to the convex ($1$-layer) linear regression problem.

Compared with previous work \citep{bartlett2018gradient, arora2018convergence} that gave convergence rate guarantees for linear neural networks, our result has several advantages:
\begin{itemize}
\item 
Our result applies to the widely used Xavier random initialization, while \citet{bartlett2018gradient} used identity initialization, and \citet{arora2018convergence} assumed that initialization is ``balanced'' and somewhat close to the global minimum.
\item Our result does not have any assumption on the input data, while \citet{bartlett2018gradient, arora2018convergence} both required whitened data.
\item Our result does not have any assumption on the global minimum, while \citet{bartlett2018gradient} assumed it to be either close to identity or positive definite, and \citet{arora2018convergence} required it to have full rank.
\end{itemize}

Our polynomial upper bound for the wide linear neural network and the exponential lower bound for the narrow linear neural network together demonstrate that \emph{width provably matters} in guaranteeing the efficiency of randomly initialized gradient descent for optimizing deep linear nets.\footnote{There are other techniques such as adaptive gradients~\citep{duchi2011adaptive, kingma2014adam} and skip-connections~\citep{he2016deep} that could help optimization. Analyses of those approaches are beyond the scope of this paper.}

\paragraph{Our Technique:}
Our proof technique is related to the recent work~\cite{arora2018convergence, arora2018optimization, du2018deep} 
which utilized a time-varying Gram matrix (or preconditioner) along the trajectory of gradient descent. 
We adopt the same idea of using such Gram matrix.
In the setting of wide linear neural networks, we carefully upper and lower bound eigenvalues of this Gram matrix throughout the optimization process, which together with some perturbation analysis implies linear convergence.
In order to establish this at initialization, we need to analyze spectral properties of product of Gaussian random matrices and show that these properties hold throughout the trajectory of gradient descent.

%% file: rel.tex
\subsection{Optimization for Deep Linear Neural Networks}
\label{subsec:compare}

\paragraph{Landscape Analysis:}
\citet{ge2015escaping,jin2017escape} showed that if an objective function satisfies that (1) all local minima are global, and (2) all saddle points are strict (i.e., there exists a negative curvature), then randomly perturbed gradient descent can escape all saddle points and find a global minimum.
Motivated by this, a series of papers~\citep{kawaguchi2016deep,hardt2016identity,lu2017depth,yun2017global,zhou2018critical,laurent2018deep} studied these landscape properties for optimizing deep linear networks.
While it was established that all local minima are global, unfortunately the strict saddle property is not satisfied even for $3$-layer linear neural networks.
Therefore, using landscape properties alone is not sufficient for proving global convergence.


\paragraph{Trajectory Analysis:}
Instead of using the indirect landscape-based approach, an alternative is to directly analyze the trajectory generated by a concrete optimization algorithm like gradient descent.
The current paper also belongs to this category.

\citet{saxe2014exact} gave a thorough empirical study on deep linear networks, showing that they exhibit some learning patterns similar to nonlinear networks.
\citet{ji2018gradient} studied the dynamics of gradient descent to optimize a deep linear neural network for classification problems, and showed that the risk converges to $0$ and the solution found is a max-margin solution.
\citet{arora2018optimization} observed that adding more layers can accelerate optimization for certain loss functions. 
\citet{du2018algorithmic} showed that using gradient descent, layers are automatically balanced.

All the above results do not show concrete convergence rates of gradient descent.
The most related papers are \cite{bartlett2018gradient} and \cite{arora2018convergence}. 
Here we give a detailed description of their results.

\citet{bartlett2018gradient} showed that if one uses identity initialization, the input data is whitened, and the target matrix is either close to identity or positive definite, then gradient descent converges to the target matrix at a linear rate.
Their result highly depends on the identity initialization scheme and has strong requirements on the input data and the target.
\citet{arora2018convergence} showed that  if the initialization is balanced and the initial loss is smaller than the loss of any low-rank solution by a margin, then gradient descent converges to global minimum at a linear rate.
However, their initialization scheme requires a special SVD step which is not used in practice, and the initial loss condition happens with exponentially small probability when the input and output dimensions are large.
Our result improves upon these two papers by (i) allowing fully random initialization, and (ii) removing all assumptions on the input data and the target.


\subsection{Optimization for Other Neural Networks}
\label{sec:other}
Many papers tried to identify the two desired geometric landscape properties of objective functions for non-linear neural networks~\citep{freeman2016topology,nguyen2017loss,venturi2018neural,soudry2016no,du2018power,soltanolkotabi2018theoretical,haeffele2017global}.
Unfortunately, these  properties do not hold even for simple non-linear shallow neural networks~\citep{yun2018small,safran2018spurious}.

A series of recent papers used trajectory-based methods to analyze gradient descent for shallow neural networks under strong data assumptions ~\citep{tian2017analytical,soltanolkotabi2017learning,brutzkus2017globally,li2017convergence,zhong2017recovery,zhang2018learning,du2018convolutional,du2018gradient}.
These results are restricted to shallow neural networks, and the assumptions are not satisfied in practice.

Recent breakthroughs were made in the optimization for extremely over-parametrized non-linear neural networks~\citep{du2018provably,du2018deep,li2018learning,allen2018convergence,zou2018stochastic}.
For deep ReLU neural networks, \citet{allen2018convergence,zou2018stochastic} showed that if the width of hidden layers is $\Omega\left(n^{30}L^{30}\log^2(\frac{1}{\epsilon})\right)$, then gradient descent converges to $\epsilon$ loss. ($n$ is the number of training samples.)
\citet{du2018deep} considered non-linear smooth activation functions like soft-plus, and showed that if the width of hidden layers is $n^4 \cdot 2^{\Omega(L)}$, then gradient descent converges to $0$ loss.\footnote{They also showed if one uses skip-connections~\cite{he2016deep}, then the width only depends polynomially on $L$. We only focus on fully-connected neural networks in this paper.}
All these results need additional assumptions on data, which also show up in the required width.
Compared with them, we have a much better bound on the required width ($L$ v.s. $L^{30}$ or $\exp(L)$), although this is not a fair comparison because linear networks are simpler than non-linear ones.
But given that we obtain a near linear dependence on depth, our result may shed light on the limit of required width in optimizing non-linear neural networks.


%% file: pre.tex
\subsection{Notation}

We use $\norm{\cdot}$ to denote the Euclidean norm of a vector or the spectral norm of a matrix, and use $\norm{\cdot}_F$ to denote the Frobenius norm of a matrix.
For a symmetric matrix, let $\lambda_{\max}(A)$ and $\lambda_{\min}(A)$ be its maximum and minimum eigenvalues, and let $\lambda_i(A)$ be its $i$-th largest eigenvalue.
Similarly, for a general matrix $B$, let $\sigma_{\max}(B)$ and $\sigma_{\min}(B)$ be its maximum and minimum singular values, and let $\sigma_i(B)$ be its $i$-th largest singular value.

Let $ I$ be the identity matrix and $[n]=\{1, 2, \ldots, n\}$.
Denote by $\cN(0, 1)$ the standard Gaussian distribution, and by $\chi^2_k$ the $\chi^2$ distribution with $k$ degrees of freedom.
Let $\cS^{d-1} = \{x\in \R^d: \norm{x}=1\}$ be the unit sphere in $\R^d$.

Let $\vectorize{A}$ be the vectorization of a matrix $A$ in column-first order.
The Kronecker product between two matrices $A \in \R^{m_A \times n_A}$ and $B \in \R^{m_B \times n_B}$ is defined as 
\small $$
A \otimes B = \left ( \begin{matrix}
a_{1,1} B& \cdots &a_{1,n_A} B \\
\vdots & \ddots & \vdots \\
a_{m_A,1} B & \cdots & a_{m_A,n_A} B 
\end{matrix} \right) \in \R^{m_A m_B \times n_A n_B},
$$ \normalsize
where $a_{i,j}$ is the element in the $(i, j)$-th entry of $A$.

We use $C$ to represent a sufficiently large universal constant throughout the paper. The specific value of $C$ can be different from line to line.

\subsection{Problem Setup}

We are given $n$ training samples $\{(x_p,y_p)\}_{p=1}^{n}\subset\R^{\din}\times\R^{\dout}$.
Let $X = \left( x_1, \ldots, x_n \right) \in \R^{\din \times n}$ be the input data matrix and $Y = \left( y_1, \ldots, y_n \right) \in \R^{\dout \times n}$ be the label matrix.

Consider the problem of training a depth-$L$ linear neural network with hidden layer width $m$ by minimizing the $\ell_2$ loss over data:
\small
\begin{equation} \label{eqn:loss-func}
	\begin{aligned}
	\ell(W_1, \ldots, W_L) &= 
	\frac{1}{2} \sum_{p=1}^n \norm{\frac{1}{\sqrt{m^{L-1} \dout}}W_L \cdots W_1 x_p - y_p}^2 \\
	\scriptstyle 
	&= \frac{1}{2} \norm{\frac{1}{\sqrt{m^{L-1} \dout}} W_L \cdots W_1 X - Y}_F^2,
	\end{aligned}
\end{equation}
\normalsize
where $W_1 \in \R^{m\times \din}, W_2, \ldots, W_{L-1} \in \R^{m\times m}$ and $W_L\in \R^{m\times \dout}$ are weight matrices to be learned.
Here $\frac{1}{\sqrt{m^{L-1}\dout}}$ is a scaling factor corresponding to Xavier initialization\footnote{We adopt this scaling factor so that we can initialize all weights from $\cN(0, 1)$.}~\citep{glorot2010understanding}, for which we provide a justification in Section~\ref{subsec:scaling}.

We consider the vanilla gradient descent (GD) algorithm for objective~\eqref{eqn:loss-func} with random initialization:
\begin{itemize}
	\item We initialize all the entries of $W_1,\ldots,W_L$ independently from $\cN(0, 1)$. Let $W_1(0), \ldots, W_L(0)$ be the weight matrices at initialization.
	
	\item Then we update the weights using GD: for $t=0, 1, 2, \ldots$ and $i\in[L]$,
	\begin{equation} \label{eqn:gd-update}
	W_i(t+1) = W_i(t) - \eta \frac{\partial \ell}{\partial W_i}(W_1(t), \ldots, W_L(t)) ,
	\end{equation}
	where $\eta>0$ is the learning rate.
\end{itemize}

For notational convenience, we denote $W_{j:i} = W_jW_{j-1}\cdots W_i$ for every $1\le i \le j \le L$.
We also define $W_{i-1:i} = I$ (of appropriate dimension) for completeness.

We use the time index $t$ for all variables that depend on $W_1, \ldots, W_L$, e.g., $W_{j:i}(t) = W_j(t) \cdots W_i(t)$, $\ell(t) = \ell(W_1(t), \ldots, W_L(t))$, etc. 

\subsection{On the Scaling Factor} \label{subsec:scaling}

The scaling factor $\frac{1}{\sqrt{m^{L-1}\dout}}$ ensures that the network at initialization preserves the size of every input in expectation.
\begin{claim} \label{claim:scaling-factor-justification}
	For any $x\in \R^{\din}$, we have $$\expect{\norm{ \frac{1}{\sqrt{m^{L-1}\dout}} W_{L:1}(0) x }^2} = \norm{x}^2.$$
\end{claim}
The proof of Claim~\ref{claim:scaling-factor-justification} is given in Appendix~\ref{app:proof-claim:scaling-factor-justification}.

%% file: result.tex
In this section we present our main result.
First note that when $m \ge \dout $ (which we will assume), the deep linear network we study has the same representation power as a linear map $x \mapsto Wx$ ($W \in \R^{\dout \times \din}$).
Hence, the optimal value $\opt$ for our objective function~\eqref{eqn:loss-func} is equal to the optimal value of the following linear regression problem:
\begin{equation} \label{eqn:linear-regression}
\opt = \min_{W \in \R^{\dout \times \din}} f(W) = \min_{W \in \R^{\dout \times \din}} \frac12 \norm{WX-Y}_F^2.
\end{equation}
Let $\Phi \in \R^{\dout \times \din}$ be a minimizer of $f$ with minimum spectral norm.\footnote{Our theorem holds for any minimizer of $f$. Since our bound improves when $\norm{\Phi}$ is smaller, we simply define $\Phi$ to be a minimum-spectral-norm minimizer.}
Let $r = \rank(X)$, and define $\kappa = \frac{\lambda_{\max}(X^\top X)}{\lambda_r(X^\top X)}$ which is the condition number of $X^\top X$.

Our main theorem is the following:
\begin{thm} \label{thm:main}
	Suppose 
	\begin{equation} \label{eqn:m-requirement}
	m \ge C \cdot L \cdot \max\left\{ r\kappa^3 \dout (1+\norm{\Phi}^2), r\kappa^3 \log\frac{r}{\delta}, \log L \right\}
	\end{equation}
	 for some $\delta\in(0, 1)$ and a sufficiently large universal constant $C>0$
and we set $\eta \le \frac{\dout}{3L\norm{X^\top X}}$.
	Then with probability at least $1- \delta$ over the random initialization, we have
	\begin{align*}
		\ell(0) - \opt &\le O \left( \max\left\{ 1, \frac{\log(r/\delta)}{\dout}, \norm{\Phi}^2  \right\} \right) \norm{X}_F^2,\\
		\ell(t) - \opt &\le \left( 1 -  \frac{\eta L\cdot \lambda_{r}(X^\top X) }{4\dout} \right)^t \left( \ell(0) - \opt \right).
	\end{align*}
\end{thm}
Theorem~\ref{thm:main} establishes that if the width of each layer is sufficiently large, randomly initialized gradient descent can reach a global minimum at a linear convergence rate.
Notably, our result is fully polynomial in the sense that we only require polynomially large width and the convergence time is also polynomial.
To our knowledge, this is the first polynomial time convergence guarantee for randomly initialized gradient descent on deep linear networks.

Ignoring logarithmic factors and assuming $\norm{\Phi}=O(1)$, our requirement on width~\eqref{eqn:m-requirement} is $m=\widetilde{\Omega}\left(Lr\kappa^3 \dout\right)$.
It remains open  whether this dependence is tight for randomly initialized gradient descent to find a global minimum in polynomial time.


In terms of convergence rate, if we set the learning rate to be $\eta = \Theta\left( \frac{\dout}{L\norm{X^\top X}} \right)$, then the predicted ratio of decrease in each iteration is $1 - \Theta\left( \frac{\lambda_{r}(X^\top X)}{\norm{X^\top X}} \right) = 1 - \Theta\left(\frac1\kappa \right)$, so the number of iterations needed to reach $\opt+\epsilon$ loss is $O\left( \kappa \log \frac1\epsilon \right)$.
This matches the convergence rate of gradient descent on the linear regression (convex!) problem~\eqref{eqn:linear-regression}.

Furthermore, notice that our requirement on the learning rate is $\eta = O(\frac{\dout}{L\norm{X^\top X}})$.
When $L=1$, this also exactly recovers the convergence result for applying gradient descent to the linear regression problem~\eqref{eqn:linear-regression}.
The reason why $L$ is in the denominator will be clear in the proof.
At a high level, we show that optimizing a deep linear network is similar to a linear regression problem with the covariance matrix being $L\cdot X^\top X$, which thus requires scaling down the learning rate by a factor of $L$.

%% file: proof-overview.tex
In this section we give an overview for the proof of Theorem~\ref{thm:main}.

First, we note that a simple reduction implies that we can make the following assumption \emph{without loss of generality}:
\begin{asmp} \label{asmp:full-rank-data}
	(Without loss of generality) $X \in \R^{\din \times r}$, $\rank(X) = r$, $Y = \Phi X$, and $\opt =0$.
\end{asmp}
See Appendix~\ref{app:full-rank-asmp} for justification.
Therefore we will work under Assumption~\ref{asmp:full-rank-data} from now on.

Now we proceed to sketch the proof of Theorem~\ref{thm:main}.
The key idea is to examine the dynamics of the network prediction on data $X$ during optimization, namely:
\begin{align*}
U = \frac{1}{\sqrt{m^{L-1} \dout}}W_{L:1} X \in \mathbb{R}^{\dout \times n}.
\end{align*}
With this notation, the network prediction at iteration $t$ is $U(t) = \frac{1}{\sqrt{m^{L-1} \dout}}W_{L:1}(t) X$, and the loss value at iteration $t$ is $\ell(t) = \frac{1}{2}\norm{U(t)-Y}_F^2$.
Hence how $U(t)$ evolves is directly related to how loss $\ell(t)$ decreases.

The gradient of our objective function~\eqref{eqn:loss-func} is
\begin{equation} \label{eqn:gradient}
\frac{\partial\ell}{\partial W_i} = \frac{1}{\sqrt{m^{L-1}\dout}}  W_{L:i+1}^\top (U - Y) \left( W_{i-1:1} X \right)^\top.
\end{equation}
Then using the update rule~\eqref{eqn:gd-update} we write
 \begin{align*}
&W_{L:1}(t+1) \\
=\,& \prod_i \left( W_i(t) - \eta \frac{\partial \ell}{\partial W_i}(t)  \right) \\
=\,& W_{L:1}(t) - \sum_{i=1}^L \eta W_{L:i+1}(t) \frac{\partial \ell}{\partial W_i}(t) W_{i-1:1}(t) + E(t)\\
=\,& W_{L:1}(t)  \\
 & -\frac{\eta}{\sqrt{m^{L-1}\dout}} \sum_{i=1}^L \Big( W_{L:i+1}(t)  W_{L:i+1}^\top(t)\\ & \quad \cdot (U(t) - Y)  \left( W_{i-1:1}(t) X \right)^\top W_{i-1:1}(t) \Big) \\& + E(t),
\end{align*} 
where $E(t)$ contains all high-order terms (those with $\eta^2$ or higher).
Multiplying this equation by $\frac{1}{\sqrt{m^{L-1} \dout}}X$ on the right we get
 \begin{align*}
U(t+1)
&= U(t) -\frac{\eta}{m^{L-1}\dout} \sum_{i=1}^L \Big( W_{L:i+1}(t)  W_{L:i+1}^\top(t)\\ &\quad \quad \cdot (U(t) - Y)  \left( W_{i-1:1}(t) X \right)^\top \left( W_{i-1:1}(t)X \right) \Big) \\&\quad + \frac{1}{\sqrt{m^{L-1} \dout}} E(t)X.
\end{align*} 
Vectorizing the above equation and using the property of Kronecker product: $\vectorize{ACB} = (B^\top \otimes A) \vectorize{C}$, we obtain
 \begin{equation} \label{eqn:U-dynamics}
\begin{aligned}
&\vectorize{U(t+1) - U(t)} \\
=\,& -\eta P(t) \cdot \vectorize{U(t)-Y} + \frac{1}{\sqrt{m^{L-1}\dout}} \vectorize{E(t)X},
\end{aligned}
\end{equation} 
where
\begin{equation} \label{eqn:P-def}
\begin{aligned}
	P(t) = \frac{1}{m^{L-1}\dout} \sum_{i=1}^L \Big[ \left( \left( W_{i-1:1}(t) X \right)^\top \left( W_{i-1:1}(t)X \right) \right)  \\ \otimes \left( W_{L:i+1}(t)  W_{L:i+1}^\top(t) \right) \Big] .
	\end{aligned}
\end{equation}
Notice that $P(t)$ is always positive semi-definite (PSD) because it is the sum of $L$ terms, each of which is the Kronecker product between two PSD matrices.

Now we assume that the high-order term $E(t)$ in~\eqref{eqn:U-dynamics} is very small (which we will rigorously prove) and ignore it for now.
Then~\eqref{eqn:U-dynamics} implies
\begin{equation} \label{eqn:U-dynamics-approx}
\vectorize{U(t+1) - Y} \approx \left( I -\eta P(t) \right)  \vectorize{U(t)-Y} .
\end{equation}
Suppose we are able to set $\eta \le \frac{1}{\lambda_{\max}(P_t)}$. Then~\eqref{eqn:U-dynamics-approx} would imply
\begin{equation*}
\norm{U(t+1) - Y}_F \le \left( 1 -\eta \lambda_{\min}(P(t)) \right)  \norm{U(t)-Y}_F.
\end{equation*}
Therefore, if we have a lower bound on $\lambda_{\min}(P(t))$ for all $t$, we will have linear convergence as desired.
We will indeed prove the following bounds on $\lambda_{\max}(P_t)$ and $\lambda_{\min}(P_t)$ for all $t$, which will essentially complete the proof:
\begin{equation} \label{eqn:P-eigenval-bounds}
\begin{aligned}
\lambda_{\max}(P_t) &\le O\left( {L\lambda_{\max}(X^\top X)}/{\dout}  \right),\\
\lambda_{\min}(P_t) &\ge \Omega\left( {L\lambda_{\min}(X^\top X)}/{\dout}  \right).
\end{aligned}
\end{equation}
We use the following approach to bound $\lambda_{\max}(P(t))$ and $\lambda_{\min}(P(t))$:
\small \begin{equation}   \label{eqn:P-max-min-eigenval-bound-method}
\begin{aligned}
&\lambda_{\max}(P(t)) \\
\le \,& \frac{1}{m^{L-1}\dout} \sum_{i=1}^L \Big[ \lambda_{\max}\left( \left( W_{i-1:1}(t) X \right)^\top \left( W_{i-1:1}(t)X \right) \right)  \\&\quad \cdot   \lambda_{\max}\left( W_{L:i+1}(t)  W_{L:i+1}^\top(t) \right)  \Big]  \\
= \,& \frac{1}{m^{L-1}\dout} \sum_{i=1}^L  \sigma_{\max}^2 \left( W_{i-1:1}(t)X \right) \cdot   \sigma_{\max}^2\left( W_{L:i+1}(t)  \right) , \\
&\lambda_{\min}(P(t)) \\
\ge \,& \frac{1}{m^{L-1}\dout} \sum_{i=1}^L \Big[  \lambda_{\min}\left( \left( W_{i-1:1}(t) X \right)^\top \left( W_{i-1:1}(t)X \right) \right)  \\&\quad \cdot   \lambda_{\min}\left( W_{L:i+1}(t)  W_{L:i+1}^\top(t) \right) \Big] \\
= \,& \frac{1}{m^{L-1}\dout} \sum_{i=1}^L \sigma_{\min}^2 \left( W_{i-1:1}(t)X \right) \cdot   \sigma_{\min}^2\left( W_{L:i+1}(t)  \right) , 
\end{aligned}
\end{equation} \normalsize
Here we have used the property that for symmetric matrices $A$ and $B$, every eigenvalue of $A \otimes B$ is the product of an eigenvalue of $A$ and an eigenvalue of $B$.
Therefore, it suffices to obtain upper and lower bounds on the singular values of $W_{i-1:1}(t)X$ and $W_{L:i+1}(t) $.
In Section~\ref{sec:init}, we establish these bounds for initialization ($t=0$).
Then we finish the proof of Theorem~\ref{thm:main} in Section~\ref{sec:proof-main}.

%% file: init.tex
In this section we establish some  properties of the weight matrices generated by random initialization.

The following lemma shows that when multiplying a fixed vector by a series of Gaussian matrices with large width, the resulting vector's norm is concentrated. 
\begin{lem} \label{lem:matrix-prod-times-vector-concentration}
	Suppose $m > C q$, and consider $q$ independent random matrices $A_1, \ldots, A_q$ $(A_1 \in \R^{m \times d}, A_2, \ldots, A_q \in \R^{m\times m})$ with i.i.d. $\cN(0, 1)$ entries.
	Then for any $v \in \cS^{d-1}$, with probability at least $1 - e^{-\Omega(m/q)}$ we have
	\begin{align*}
	0.9 m^{q/2} \le \norm{A_q\cdots A_1v} \le 1.1 m^{q/2}.
	\end{align*}
\end{lem}
\begin{proof}
	See Appendix~\ref{app:proof-lem:matrix-prod-times-vector-concentration}.
\end{proof}

The next three propositions show the key properties of products of weight matrices at initialization.

\begin{prop}\label{prop:initial-left-product-concentration}
	For any $1 < i \le L$, with probability at least $1 - e^{-\Omega\left( m/L \right)}$ we have
	\begin{align*}
	\sigma_{\max}(W_{L:i}(0)) &\le 1.2 m^{\frac{L-i+1}{2}} ,\\
	 \sigma_{\min}\left( W_{L:i}(0) \right)  &\ge 0.8  m^{\frac{L-i+1}{2}}.
	\end{align*}
\end{prop}
\begin{proof}
	Let $A = W_{L:i}^\top(0)$. Since $A \in \R^{m\times \dout}$ and $m>\dout$, we know $\norm{A} = \sup_{v\in \cS^{\dout-1}} \norm{Av}$ and $\sigma_{\min}(A) = \inf_{v\in \cS^{\dout-1}} \norm{Av}$.
	Also, from Lemma~\ref{lem:matrix-prod-times-vector-concentration} we know that for any fixed $v\in \cS^{\dout-1}$, with probability at least $1 - e^{-\Omega(m/L)}$ we have $\norm{Av} \in \left[ 0.9 m^{\frac{L-i+1}{2}}, 1.1 m^{\frac{L-i+1}{2}} \right]$.
	
	The rest of the proof is by a standard $\eps$-net argument.
	Let $\eps = 0.01$.
	Take an $\eps$-net $\cN$ for $\cS^{\dout-1}$ with $\abs{\cN} \le (3/\eps)^{\dout}$. By a union bound, with probability at least $1 - \abs{\cN} e^{-\Omega(m/L)}$, for all $u\in \cN$ simultaneously we have $\norm{Au}/\norm{u} \in \left[ 0.9 m^{\frac{L-i+1}{2}}, 1.1 m^{\frac{L-i+1}{2}} \right]$.
	Suppose this happens for every $u\in \cN$.
	Next, for any $v\in \cS^{\dout-1}$, there exists $u \in \cN$ such that $\norm{u-v} \le \eps$. Then we have
	\begin{align*}
	\norm{Av} &\le \norm{Au} + \norm{A(u-v)}
	\le 1.1 m^{\frac{L-i+1}{2}} \norm{u} + \eps\norm{A} \\
	&\le 1.1 (1+\eps) m^{\frac{L-i+1}{2}}  + \eps\norm{A} .
	\end{align*}
	Taking supreme over $v\in \cS^{\dout-1}$, we obtain
	\begin{align*}
	\norm{A}  \le \frac{1.1 (1+\eps) m^{\frac{L-i+1}{2}}}{1-\eps}
	\le 1.2 m^{\frac{L-i+1}{2}}.
	\end{align*}
	For the lower bound, we have
	 \begin{align*}
	\norm{Av} &\ge \norm{Au} - \norm{A(u-v)}
	\ge 0.9 m^{\frac{L-i+1}{2}} \norm{u} - \eps\norm{A} \\
	&\ge 0.9 (1-\eps) m^{\frac{L-i+1}{2}} - \eps \cdot 1.2 m^{\frac{L-i+1}{2}}\\
	&\ge 0.8 m^{\frac{L-i+1}{2}}.
	\end{align*} 
	Taking infimum over $v\in \cS^{\dout-1}$ we get $\sigma_{\min}(A) \ge 0.8 m^{\frac{L-i+1}{2}}$.
	
	The success probability is at least $1 - \abs{\cN} e^{-\Omega(m/L)} = 1 - e^{-\Omega(\frac mL) + \dout \log(\frac3\eps)} = 1 - e^{-\Omega(\frac mL)}$
	since $m>CL\dout$.
\end{proof}

\begin{prop}\label{prop:initial-right-product-concentration}
	For any $1 \le i < L$, with probability at least $1 - e^{-\Omega\left( m/L \right)}$ we have
	\begin{align*}
	\sigma_{\max}(W_{i:1}(0) \cdot X) &\le 1.2 m^{\frac{i}{2}} \sigma_{\max}(X), \\
	\sigma_{\min}\left( W_{i:1}(0) \cdot X \right) & \ge 0.8  m^{\frac{i}{2}}   \sigma_{\min}(X).
	\end{align*}
\end{prop}
\begin{proof}
	The proof is similar to the proof of Proposition~\ref{prop:initial-left-product-concentration} and is deferred to Appendix~\ref{app:proof-prop:initial-right-product-concentration}.
\end{proof}

\begin{prop} \label{prop:initial-middle-product-concentration}
	For any $1< i \le j <L $, with probability at least $1 -  e^{-\Omega\left( m/L \right)}$ we have 
	\begin{align*}
	\norm{W_{j:i}(0)} \le O\left( \sqrt L m^{\frac{j-i+1}{2}}\right).
	\end{align*}
\end{prop}

\begin{proof}
	Let $A = W_{j:i}(0)$.
	From Lemma~\ref{lem:matrix-prod-times-vector-concentration} we know that for any fixed $v\in \cS^{m-1}$, with probability at least $1 - e^{-\Omega(m/L)}$ we have $\norm{Av} \in \left[ 0.9 m^{\frac{j-i+1}{2}}, 1.1 m^{\frac{j-i+1}{2}} \right]$.
	
	Take a small constant $c>0$ and partition the index set $[m]$ into $[m] = S_1 \cup S_2 \cup \cdots \cup S_{L/c}$ where $\abs{S_l} = cm/L$ for each $l$.
	For each $l$, taking a $\frac12$-net $\cN_l$ for all the unit vectors supported in $S_l$, i.e., a $\frac12$-net for the set $V_{S_l} = \left\{ v\in \cS^{m-1}: \supp(v) \subseteq S_l \right\}$, we know that
	\begin{equation} \label{eqn:subspace-op-norm-bound}
	\norm{Au} \le O\left( m^{\frac{j-i+1}{2}} \right), \quad \forall u\in V_{S_l},
	\end{equation}
	with probability at least $1 - \abs{\cN_l} e^{-\Omega(m/L)} = 1 - e^{-\Omega(m/L) + ({cm}/{L})\log 6} =1 - e^{-\Omega(m/L)} $.
	Then taking a union bound over all $l$, we know that~\eqref{eqn:subspace-op-norm-bound} holds for all $l$ simultaneously with probability at least $1 - \frac Lc e^{-\Omega(m/L)}$.
	Conditioned on this, for any $v \in \R^m$, we can partition its coordinates and write it as the sum $v = \alpha_1 v_1 + \alpha_2 v_2 + \cdots + \alpha_{L/c} v_{L/c}$ where $\alpha_l\in\R$ and $v_l \in V_{S_l}$ for each $l$.
	Then we have
	 \begin{align*}
	\norm{Av} 
	&\le \sum\nolimits_l \norm{A\cdot \alpha_l v_l}
	\le \sum\nolimits_l \abs{\alpha_l} O\left( m^{\frac{j-i+1}{2}} \right) \\
	&\le O\left( m^{\frac{j-i+1}{2}} \right) \sqrt{\frac Lc \sum\nolimits_l \alpha_l^2}
	= O\left(\sqrt L m^{\frac{j-i+1}{2}} \right) \norm{v}.
	\end{align*} 
	This means $\norm{A} \le O\left(\sqrt L m^{\frac{j-i+1}{2}} \right) $.
	The success probability is at least $1 - \frac Lc e^{-\Omega(m/L)} = 1 - e^{-\Omega(m/L)+ \log(L/c)} = 1 - e^{-\Omega(m/L)}$ since $m> C L \log L$.
\end{proof}

To close this section, we bound the loss value $\ell(0)$ at initialization, which proves the first part of Theorem~\ref{thm:main}.

\begin{prop} \label{prop:init-loss-bound}
	With probability at least $1 - e^{-\Omega(m/L)} - \delta/2$, we have
	$\ell(0) \le O \left( \max\left\{ 1, \frac{\log(r/\delta)}{\dout}, \norm{\Phi}^2  \right\} \right) \norm{X}_F^2$. 
\end{prop}
\begin{proof}
	See Appendix~\ref{app:proof-prop:init-loss-bound}.
\end{proof}

%% file: proof_main.tex
In this section we prove Theorem~\ref{thm:main} based on ingredients from Sections~\ref{sec:proof-overview} and~\ref{sec:init}.

From Propositions~\ref{prop:initial-left-product-concentration}, \ref{prop:initial-right-product-concentration}, \ref{prop:initial-middle-product-concentration} and~\ref{prop:init-loss-bound}, we know
 that with probability at least $1- L^2 e^{-\Omega(m/L)} - \delta/2 \ge 1 - \delta$, 
 the following conditions of initialization are satisfied simultaneously:
 
\begin{align} \label{eqn:initial-properties-summary}
\begin{cases}
&\sigma_{\max}(W_{L:i}(0)) \le 1.2 m^{\frac{L-i+1}{2}}, \quad \forall 1<i\le L, \\
&\sigma_{\min}\left( W_{L:i}(0) \right) \ge 0.8  m^{\frac{L-i+1}{2}} , \quad \forall 1<i\le L, \\
&\sigma_{\max}(W_{i:1}(0) \cdot X) \le 1.2 m^{\frac{i}{2}} \sigma_{\max}(X), \quad \forall 1\le i < L, \\
&\sigma_{\min}\left( W_{i:1}(0) \cdot X \right)  \ge 0.8  m^{\frac{i}{2}}   \sigma_{\min}(X), \quad \forall 1\le i < L, \\
&	\norm{W_{j:i}(0)} \le O\left( \sqrt L m^{\frac{j-i+1}{2}}\right), \quad \forall 1<i\le j <L,\\
&\ell(0) \le B. 
\end{cases}
\end{align}
Here we define $B = O \left( \max\left\{ 1, \frac{\log(r/\delta)}{\dout}, \norm{\Phi}^2  \right\} \right) \norm{X}_F^2$ which is the upper bound on $\ell(0)$ from Proposition~\ref{prop:init-loss-bound}.

From our requirement on $m$ \eqref{eqn:m-requirement}, we know
 \begin{equation} \label{eqn:m-bound-implication}
\begin{aligned}
m &\ge C \cdot L r\kappa^3 \max\left\{ \dout (1+\norm{\Phi}^2), \log\frac{r}{\delta} \right\} \\
&\ge \frac{C  L r\kappa^3 \dout B}{\norm{X}_F^2}
\ge \frac{C  L r\kappa^3 \dout B}{r \norm{X}^2}
= \frac{C  L \norm{X}^4 \dout B}{\sigma_{\min}^6(X)}.
\end{aligned}
\end{equation} 

Now we establish our convergence result conditioned on all properties in~\eqref{eqn:initial-properties-summary}.
Specifically, we use induction on $t$ to simultaneously prove the following three properties $\cA(t)$, $\cB(t)$ and $\cC(t)$ for all $t=0, 1, \ldots$:
\begin{itemize}
	\item $\cA(t)$:
	\vspace{-15pt}
		\begin{align*}
		\ell(t) &\le \left(  1 - \frac14 \eta L \sigma_{\min}^2(X) / \dout \right)^t \ell(0)  \\
		&  \le \left(  1 - \frac14 \eta L \sigma_{\min}^2(X) /\dout \right)^t   B.
		\end{align*} 
	\vspace{-20pt}
	\item $\cB(t)$:
	\vspace{-6pt}
	\begin{align*}
	\begin{cases}
	 	&\sigma_{\max}(W_{L:i}(t)) \le \frac54 m^{\frac{L-i+1}{2}}, \, \forall 1<i\le L, \\
	 	 &\sigma_{\min}\left( W_{L:i}(t) \right) \ge \frac34  m^{\frac{L-i+1}{2}} , \, \forall 1<i\le L, \\
	 	 &\sigma_{\max}(W_{i:1}(t) \cdot X) \le \frac54 m^{\frac{i}{2}} \sigma_{\max}(X), \, \forall 1\le i < L, \\
	 	 &\sigma_{\min}\left( W_{i:1}(t) \cdot X \right)  \ge \frac34  m^{\frac{i}{2}}   \sigma_{\min}(X), \, \forall 1\le i < L, \\
	 	 &	\norm{W_{j:i}(t)} \le O\left( \sqrt L m^{\frac{j-i+1}{2}}\right), \, \forall 1<i\le j <L.
	 	 \end{cases}
	\end{align*}
	\vspace{-10pt}
	\item $\cC(t)$:
	\vspace{-10pt}
	 \begin{align*}
	\norm{W_i(t) - W_i(0)}_F \le  \frac{24\sqrt{B\dout}\norm{X}}{ L \sigma_{\min}^2(X)}, \quad \forall i\in[L].
	\end{align*} 
\end{itemize}

Notice that if we prove $\cA(t)$ for all $t\ge 0$, we will finish the proof of Theorem~\ref{thm:main}.

The initial conditions $\cA(0)$ and $\cB(0)$ follow directly from~\eqref{eqn:initial-properties-summary}, and $\cC(0)$ is trivially true.
In order to establish $\cA(t)$, $\cB(t)$ and $\cC(t)$ for all $t$, in Sections~\ref{subsec:induction-1}-\ref{subsec:induction-3} we will prove respectively the following claims for all $t\ge0$:
\begin{claim} \label{claim:induction-1}
	$\cA(0), \ldots, \cA(t), \cB(0), \ldots, \cB(t) \Longrightarrow \cC(t+1)$.
\end{claim}
\begin{claim}\label{claim:induction-2}
	$\cC(t) \Longrightarrow \cB(t)$.
\end{claim}
\begin{claim}\label{claim:induction-3}
	$\cA(t), \cB(t) \Longrightarrow \cA(t+1)$.
\end{claim}
The proof of Theorem~\ref{thm:main} is finished after the above three claims are proved.

\subsection{Proof of Claim~\ref{claim:induction-1}}
\label{subsec:induction-1}

Denote $\gamma = \frac14 L \sigma_{\min}^2(X)/\dout$.
From $\cA(0), \ldots, \cA(t)$ we know $\ell(s) \le (1-\eta\gamma)^sB$ for all $0\le s \le t$.

From the gradient expression~\eqref{eqn:gradient}, for all $0\le s\le t$ and all $i\in[L]$ we can bound:
 \begin{equation} \label{eqn:gradient-norm-bound}
\begin{aligned}
&\norm{\frac{\partial\ell}{\partial W_i}(s)}_F\\
\le\,& \frac{1}{\sqrt{m^{L-1}\dout}}  \norm{W_{L:i+1}(s)} \norm{U(s) - Y}_F \norm{W_{i-1:1}(s) X}\\
\le\,& \frac{1}{\sqrt{m^{L-1}\dout}} \cdot \frac54 m^{\frac{L-i}{2}}  \cdot \sqrt{2\ell(s)} \cdot \frac54 m^{\frac{i-1}{2}} \norm{X}\\
\le\,& \frac{3\sqrt{(1-\eta\gamma)^sB}}{\sqrt{\dout}}   \norm{X} , 
\end{aligned}
\end{equation} 
where we have used $\cB(s)$.

%

Then we can bound $\norm{W_i(t+1) - W_i(0)}_F$ for all $i\in[n]$:
 \begin{align*}
&\norm{W_i(t+1) - W_i(0)}_F
\le \sum_{s=0}^{t} \norm{W_i(s+1) - W_i(s)}_F\\
=\,&  \sum_{s=0}^{t} \norm{\eta \frac{\partial\ell}{\partial W_i}(s)}_F
\le \eta \sum_{s=0}^{t} \frac{3\sqrt{(1-\eta\gamma)^sB}}{\sqrt{\dout}}   \norm{X} \\
\le\,& \frac{3\eta \sqrt{B}}{\sqrt{\dout}} \norm{X} \sum_{s=0}^{t-1}(1-\eta\gamma/2)^s    
\le \frac{3 \eta \sqrt{B}}{ \sqrt{\dout}} \norm{X} \cdot \frac{2}{\eta\gamma}     \\
=\,& \frac{24 \sqrt{B \dout} \norm{X}}{ L \sigma_{\min}^2(X)} .
\end{align*} 
This proves $\cC(t+1)$.

\subsection{Proof of Claim~\ref{claim:induction-2}}
\label{subsec:induction-2}

Let $R = \frac{24\sqrt{B\dout}\norm{X}}{ L \sigma_{\min}^2(X)}$ and denote $\Delta_i = W_i(t) - W_i(0)$ $(i\in[L])$.
Then using $\norm{\Delta_i}_F \le R\, (\forall i\in[L])$ we will show the followings:
\begin{align}
&\norm{W_{L:i}(t) - W_{L:i}(0)} \le 0.05 m^{\frac{L-i+1}{2}}, \, \forall 1<i\le L, \label{eqn:perturbation-left} \\
&\norm{(W_{i:1}(t) - W_{i:1}(0))X} \le 0.05 m^{\frac{i}{2}} \sigma_{\min}(X), \, \forall 1\le i< L, \label{eqn:perturbation-right} \\
&\norm{W_{j:i}(t) - W_{j:i}(0)} \le 0.05 \sqrt L m^{\frac{j-i+1}{2}},  \forall 1<i \le j < L. \label{eqn:perturbation-middle}
\end{align}
Combing them with~\eqref{eqn:initial-properties-summary}, we will finish the proof of $\cB(t)$.

First we prove~\eqref{eqn:perturbation-middle}.
For $1\le i < j \le L$, we can write
\begin{align*}
W_{j:i}(t) = \left( W_j(0) + \Delta_j \right) \cdots \left( W_i(0) + \Delta_i \right) .
\end{align*}
Expanding the above product,
each term except the leading term $W_{j:i}(0)$ has the form:
\begin{equation} \label{eqn:perturbation-terms}
\begin{aligned}
W_{j:(k_s+1)}(0)	\cdot \Delta_{k_s} \cdot W_{(k_s-1):(k_{s-1}+1)}(0) \cdot \Delta_{k_{s-1}} \cdots \\ \cdot \Delta_{k_1} \cdot W_{(k_1-1):i}(0)  ,
\end{aligned}
\end{equation}
where $i \le k_1 < \cdots < k_s \le j$ are positions at which perturbation terms $\Delta_{k_l}$ are taken out, and at any other position $k$, $W_k(0)$ is used.
Note that every factor in~\eqref{eqn:perturbation-terms} of the form $W_{b:a}(0)$ satisfies $\norm{W_{b:a}(0)} \le O\left(\sqrt Lm^{\frac{b-a+1}{2}}\right)$ because of~\eqref{eqn:initial-properties-summary};
there are at most $s+1$ such factors, so the product of their spectral norms is at most $\left(O(\sqrt{L})\right)^{s+1} m^{\frac{j-i+1-s}{2}} $.
Thus, we can bound the sum of all terms like~\eqref{eqn:perturbation-terms} as
 \begin{equation} \label{eqn:perturbation-bound-middle}
\begin{aligned}
&\norm{W_{j:i}(t) - W_{j:i}(0)} \\
\le\,& \sum_{s=1}^{j-i+1} \binom{j-i+1}{s} R^s \left( O\left(\sqrt L\right) \right)^{s+1} m^{\frac{j-i+1-s}{2}} \\
\le\,&  \sum_{s=1}^{j-i+1} L^s R^s \left( O\left(\sqrt L\right) \right)^{s+1} m^{\frac{j-i+1-s}{2}} \\
=\,& O\left(\sqrt L\right) \sum_{s=1}^{j-i+1} \left( \frac{ O\left(L^{3/2}R\right)}{\sqrt m} \right)^s   m^{\frac{j-i+1}{2}} \\
\le\,& 0.05 \sqrt{L} m^{\frac{j-i+1}{2}} ,
\end{aligned}
\end{equation} 
as long as $m > C L^3 R^2$ which is implied by~\eqref{eqn:m-bound-implication}.
This proves~\eqref{eqn:perturbation-middle}.

The proof of~\eqref{eqn:perturbation-left} is very similar.
We still look at products of the form~\eqref{eqn:perturbation-terms} with $j=L$.
Still, there are at most $s+1$ factors of the form $W_{b:a}(0)$ each satisfying $\norm{W_{b:a}(0)} \le O\left(\sqrt Lm^{\frac{b-a+1}{2}}\right)$. However, there are at most $s$ (instead of $s+1$) such factors such that $1<a\le b<L$.
Therefore, the product of spectral norms of all such factors is at most $\left(O(\sqrt{L})\right)^{s} m^{\frac{L-i+1-s}{2}} $.
In other words, we can save a factor of $O(\sqrt L)$ compared with~\eqref{eqn:perturbation-bound-middle}. This gives us
\begin{align*}
\norm{W_{L:i}(t) - W_{L:i}(0)} \le 0.05 m^{\frac{L-i+1}{2}},
\end{align*}
proving~\eqref{eqn:perturbation-left}.

Next we prove~\eqref{eqn:perturbation-right}.
For $1\le i<L$, we need to bound the sum of terms of the following form:
 \begin{equation*} 
\begin{aligned}
W_{i:(k_s+1)}(0)	\cdot \Delta_{k_s} \cdot W_{(k_s-1):(k_{s-1}+1)}(0) \cdot \Delta_{k_{s-1}} \cdots \\ \cdot \Delta_{k_1} \cdot W_{(k_1-1):1}(0) X .
\end{aligned}
\end{equation*}
Again, similar as before and noting $\norm{W_{(k_1-1):1}(0) X} \le \frac54 m^{\frac{k_1-1}{2}}\norm{X}$, we have
 \small \begin{equation*} 
\begin{aligned}
&\norm{W_{i:1}(t) - W_{i:1}(0)} \\
\le\,& \sum_{s=1}^{i} \binom{j-i+1}{s} R^s \left( O\left(\sqrt L\right) \right)^{s} \cdot \frac54  m^{\frac{i-s}{2}} \norm{X} \\
\le\,& \frac54 \sum_{s=1}^{i} L^s R^s \left( O\left(\sqrt L\right) \right)^{s} m^{\frac{i-s}{2}} \norm{X} \\
=\,& \frac54 m^{\frac i2} \sum_{s=1}^{i} \left( \frac{ O\left(L^{3/2}R\right)}{\sqrt m} \right)^s \norm{X}  \\
\le\,& 0.05 m^{\frac{i}{2}} \sigma_{\min}(X) ,
\end{aligned}
\end{equation*} \normalsize
as long as $ m > C L^{3} R^2 \cdot \frac{\norm{X}^2}{\sigma_{\min}^2(X)} = CL^3R^2 \kappa$ which is implied by~\eqref{eqn:m-bound-implication}.
This finishes the proof of~\eqref{eqn:perturbation-right}.

\subsection{Proof of Claim~\ref{claim:induction-3}}
\label{subsec:induction-3}

Recall that in Section~\ref{sec:proof-overview} we derived~\eqref{eqn:U-dynamics} which is the main equation to establish convergence.
In order to establish convergence from~\eqref{eqn:U-dynamics} we need to prove upper and lower bounds~\eqref{eqn:P-eigenval-bounds} on the eigenvalues of $P(t)$, as well as show that the high-order term $E(t)$ is small.
We will prove these using $\cB(t)$.


Directly from~\eqref{eqn:P-max-min-eigenval-bound-method} and $\cB(t)$, we have
\small \begin{align*}
\lambda_{\max}(P(t)) 
&\le \frac{1}{m^{L-1}\dout} \sum_{i=1}^L \left( \frac54 m^{\frac{i-1}{2}} \sigma_{\max}(X) \right)^2 \left( \frac54 m^{\frac{L-i}{2}} \right)^2 \\
&\le 3 L \sigma_{\max}^2(X) / \dout, \\
\lambda_{\min}(P(t)) 
&\ge \frac{1}{m^{L-1}\dout} \sum_{i=1}^L \left( \frac34 m^{\frac{i-1}{2}} \sigma_{\min}(X) \right)^2 \left( \frac34 m^{\frac{L-i}{2}} \right)^2 \\
&\ge \frac{3}{10} L \sigma_{\min}^2(X) / \dout.
\end{align*} \normalsize


In Appendix~\ref{app:proof-claim:induction-3 (cntd)}, we will prove the following bound on the high-order terms in~\eqref{eqn:U-dynamics}:
\[
\frac{1}{\sqrt{m^{L-1}\dout}} \norm{E(t) X}_F \le  \frac16 \eta \lambda_{\min}(P_t) \norm{U(t)-Y}_F .
\]
Finally, from~\eqref{eqn:U-dynamics} and $\eta \le \frac{\dout}{3L \norm{X}^2} \le \frac{1}{\lambda_{\max}(P_t)}$ we have
\small
\begin{align*}
&\norm{U(t+1)-Y}_F= \norm{\vectorize{U(t+1) - Y} }\\
=\,& \norm{ (I-\eta P(t)) \cdot \vectorize{U(t)-Y} + \frac{1}{\sqrt{m^{L-1}\dout}} \vectorize{E(t)X}} \\
\le\,& (1 - \eta \lambda_{\min}(P(t)) ) \norm{\vectorize{U(t)-Y}} + \frac{1}{\sqrt{m^{L-1}\dout}} \norm{E(t) X}_F \\
\le\,& (1 - \eta \lambda_{\min}(P(t)) ) \norm{U(t)-Y}_F + \frac16 \eta \lambda_{\min}(P_t) \norm{U(t)-Y}_F \\
=\,& \left(1 - \frac56 \eta \lambda_{\min}(P(t)) \right) \norm{U(t)-Y}_F  \\
=\,& \left(1 - \frac14 \eta L \sigma_{\min}^2(X)/\dout \right) \norm{U(t)-Y}_F .
\end{align*}
\normalsize
This implies $\ell(t+1) \le \left(1 - \frac14 \eta L \sigma_{\min}^2(X)/\dout \right)^2 \ell(t) \le \left(1 - \frac14 \eta L \sigma_{\min}^2(X)/\dout \right)\ell(t)$.
Combined with $\cA(t)$, this proves $\cA(t+1)$.


%% file: conclusion.tex
We prove that gradient descent with random initialization converges to a global minimum of the $\ell_2$ loss on a wide deep linear neural network. The required width in hidden layers is near linear  in the depth of the network.
This result improves upon previous convergence results for deep linear networks, and demonstrates that adding width can eliminate the known exponential curse of depth in linear networks.
Our result may shed light on the required width in non-linear neural networks.

%% file: ack.tex
 SSD acknowledges support from AFRL grant FA8750- 17-2-0212 and DARPA D17AP00001.
 WH acknowledges support from NSF, ONR, Simons Foundation, Schmidt Foundation, Mozilla Research, Amazon Research, DARPA and SRC.

%% file: appendix.tex
\section{Proof of Claim~\ref{claim:scaling-factor-justification}}
\label{app:proof-claim:scaling-factor-justification}

\begin{proof}[Proof of Claim~\ref{claim:scaling-factor-justification}]
	First, it is easy to see that for a random matrix $A\in \R^{d_2 \times d_1}$ with i.i.d. $\cN(0, 1)$ entries and any vector $v\in \R^{d_1}\setminus\{0\}$, the distribution of $\frac{\norm{Av}^2}{\norm{v}^2}$ is $\chi^2_{d_2}$.
	We rewrite ${\norm{W_{L:1}(0)x}^2}/{\norm{x}^2}$ as
	\begin{align*}
	{\norm{W_{L:1}(0)x}^2}/{\norm{x}^2} = Z_L Z_{L-1}\cdots Z_1,
	\end{align*}
	where $Z_i = \frac{\norm{W_{i:1}(0)x}^2}{\norm{W_{i-1:1}(0)x}^2} = \frac{\norm{W_i(0)W_{i-1:1}(0)x}^2}{\norm{W_{i-1:1}(0)x}^2}$.
	Then we have $Z_1 \sim \chi^2_m$, $Z_i | (Z_1, \ldots, Z_{i-1}) \sim \chi^2_m$ ($1<i<L$) and $Z_L | (Z_1, \ldots, Z_{L-1}) \sim \chi^2_{\dout}$.
	Therefore, $Z_1, \ldots, Z_L$ are independent $\chi^2$ random variables.
	It follows that
	\begin{align*}
	&\expect{ {\norm{W_{L:1}(0)x}^2}/{\norm{x}^2}} 
	= \expect{Z_L Z_{L-1}\cdots Z_1} \\
	=\,& \prod\nolimits_{i=1}^L \expect{Z_i}
	= m^{L-1} \dout.\qedhere
	\end{align*} 
\end{proof}

\section{Justification of Assumption~\ref{asmp:full-rank-data}}
\label{app:full-rank-asmp}


We have the following claim:
\begin{claim} \label{claim:equivalent-objective}
	There exists a matrix $\bar{X} \in \R^{\din \times r}$ of rank $r$ such that $\bar X \bar X^\top = XX^\top$ and that
	the objective~\eqref{eqn:linear-regression} can be rewritten as
	\begin{align*}
	f(W) = \frac12 \norm{ W\bar X - \Phi \bar X }_F^2 + \opt.
	\end{align*}
\end{claim}
\begin{proof}[Proof of Claim~\ref{claim:equivalent-objective}]
	Since $\Phi \in \argmin_{W \in \R^{\dout \times \din}} f(W)$, we have
	\begin{align*}
	\nabla f(\Phi) = (\Phi X - Y) X^\top = 0.
	\end{align*}
	Then for any $W \in \R^{\dout \times \din}$ we have
	\begin{align*}
	&f(W) \\
	=\,& \frac{1}{2} \norm{ (WX - \Phi X) + (\Phi X - Y) } \\
	=\, & \frac12 \norm{ (W - \Phi) X }_F^2 + \langle (W - \Phi) X, \Phi X - Y \rangle + \frac12 \norm{\Phi X - Y}_F^2 \\
	=\, & \frac12 \norm{ (W - \Phi) X }_F^2 + \langle W - \Phi , \nabla f(\Phi) \rangle + \opt \\
	=\, & \frac12 \norm{ (W - \Phi) X }_F^2 +  \opt \\
	=\, & \frac12 \tr\left( (W - \Phi) XX^\top (W - \Phi) \right) +  \opt.
	\end{align*}
	Since $\rank(XX^\top) = \rank(X) = r$, there exists $\bar X \in \R^{\din \times r}$ such that $\bar X \bar X^\top = XX^\top$.
	For such $\bar X$, we have
	\begin{align*}
	f(W) &= \frac12 \tr\left( (W - \Phi) \bar X \bar X^\top (W - \Phi) \right) +  \opt \\
	&= \frac12 \norm{ (W - \Phi) \bar X }_F^2 +  \opt. \qedhere
	\end{align*}
\end{proof}

Claim~\ref{claim:equivalent-objective} indicates that if we replace the data $(X, Y)$ with $(\bar X, \Phi \bar X)$, the objective function only changes by an offset $\opt$.
In particular, this does not affect the gradient descent dynamics~\eqref{eqn:gd-update} at all.
Furthermore, since $X$ and $\bar X$ have the same rank and spectrum, this also does not change the statement of Theorem~\ref{thm:main} (except that $\opt$ becomes $0$).
Therefore, we can make this change without loss of generality, leading to Assumption~\ref{asmp:full-rank-data}.



\section{Proof of Lemma~\ref{lem:matrix-prod-times-vector-concentration}}
\label{app:proof-lem:matrix-prod-times-vector-concentration}

\begin{proof}[Proof of Lemma~\ref{lem:matrix-prod-times-vector-concentration}]
	Same as the proof of Claim~\ref{claim:scaling-factor-justification}, we know that $\norm{A_q\cdots A_1v}^2$ has the same distribution as 
	\begin{align*}
	Z_q Z_{q-1} \cdots Z_1,
	\end{align*}
	where $Z_1, \ldots, Z_q$ are i.i.d. from $\chi^2_m$.
	Therefore it suffices to bound $Z_q Z_{q-1} \cdots Z_1$.
	
	Recall the moments of $Z\sim \chi^2_m$:
	\begin{align*}
	\expect{Z^\lambda} = \frac{2^\lambda \Gamma\left( \frac m2 + \lambda \right)}{\Gamma\left( \frac m2 \right)}, \qquad \forall \lambda >  -\frac m2,
	\end{align*}
	where $\Gamma(\cdot)$ is the gamma function.
	Letting $\lambda = \alpha m$ ($\alpha \ge -\frac14$) and using Stirling's approximation for the gamma function, we have
	\begin{equation} \label{eqn:chi-square-moment-bound}
	\begin{aligned}
	\expect{Z^\lambda} 
	&= \frac{2^\lambda \sqrt{\frac{2\pi}{\frac m2+\lambda}} \left( \frac{\frac m2 + \lambda}{e} \right)^{\frac m2 + \lambda} \left(1+O\left(\frac1m\right) \right)}{\sqrt{\frac{2\pi}{\frac m2}} \left( \frac{\frac m2}{e} \right)^{\frac m2} \left(1+O\left(\frac1m\right) \right)  } \\
	&= (1+2\alpha)^{-1/2} m^{\alpha m} \exp\left\{ \left( \left(\alpha+\frac12\right)\log(1+2\alpha)-\alpha \right)m \right\} \cdot \left(1+O\left(\frac1m\right) \right) \\
	&\le (1+2\alpha)^{-1/2} m^{\alpha m} \exp\left( 2\alpha^2 m \right) \cdot \left(1+O\left(\frac1m\right) \right) \\
	&= \exp\left(-\frac12 \log\left(1+\frac{2\lambda}{m}\right) + \lambda \log m  + \frac{2\lambda^2}{m} \right) \cdot \left(1+O\left(\frac1m\right) \right) , \qquad \forall \lambda \ge -\frac14 m.
	\end{aligned}
	\end{equation}
	Here we have used $\left(\alpha+\frac12\right)\log(1+2\alpha)-\alpha \le 2\alpha^2 \,(\forall \alpha >  - 1/2)$.
	
	Now we use~\eqref{eqn:chi-square-moment-bound} to obtain tail bounds for $Z_q\cdots Z_1$.
	
	For the upper tail, consider a small constant $c>0$. For any $\lambda>0$ we have
	\begin{align*}
	\Pr\left[Z_q\cdots Z_1 > e^c m^q \right]
	&= \Pr\left[(Z_q\cdots Z_1)^\lambda > e^{\lambda c} m^{\lambda q} \right] \\
	&\le e^{-\lambda c} m^{-\lambda q} \expect{(Z_q\cdots Z_1)^\lambda} \\
	&= e^{-\lambda (c+q\log m)} \prod_{i=1}^q \expect{Z_i^\lambda} \\
	&\le e^{-\lambda (c+q\log m)} \left(    \exp\left(-\frac12 \log\left(1+\frac{2\lambda}{m}\right) + \lambda \log m  + \frac{2\lambda^2}{m} \right) \cdot \left(1+O\left(\frac1m\right) \right)  \right)^q \\
	&\le \exp\left( -\lambda (c+q\log m) + \lambda q \log m +  \frac{2\lambda^2q}{m} \right) \cdot \left(1+O\left(\frac1m\right) \right)^q \\
	&\le \exp\left( -\lambda c +  \frac{2\lambda^2q}{m} \right) \cdot \left(1+O\left(\frac1m\right) \right)^q.
	\end{align*}
	Since $m>Cq$, we have $\left(1+O\left(\frac1m\right) \right)^q = O(1)$. Letting $c = \log(1.1)$ and $\lambda = \frac{cm}{4q}$, the above inequality becomes
	\begin{equation} \label{eqn:chi-square-prod-upper-tail}
	\Pr\left[Z_q\cdots Z_1 > 1.1 m^q \right] \le O\left( e^{-\frac{\log^2(1.1) \cdot m}{8q}} \right) = e^{-\Omega(m/q)}.
	\end{equation}
	
	Similarly, for the lower tail, consider a constant $c>0$. For any $-\frac m4 < \lambda <0 $ we have
	\begin{align*}
	\Pr\left[Z_q\cdots Z_1 < e^{-c} m^q \right]
	&= \Pr\left[(Z_q\cdots Z_1)^\lambda > e^{-\lambda c} m^{\lambda q} \right] \\
	&\le e^{\lambda c} m^{-\lambda q} \expect{(Z_q\cdots Z_1)^\lambda} \\
	&= e^{\lambda (c-q\log m)} \prod_{i=1}^q \expect{Z_i^\lambda} \\
	&\le e^{\lambda (c-q\log m)} \left(    \exp\left(-\frac12 \log\left(1+\frac{2\lambda}{m}\right) + \lambda \log m  + \frac{2\lambda^2}{m} \right) \cdot \left(1+O\left(\frac1m\right) \right)  \right)^q \\
	&\le \exp\left( \lambda (c-q\log m) - \frac12 q \log\left(1+\frac{2\lambda}{m}\right)  + \lambda q \log m +  \frac{2\lambda^2q}{m} \right) \cdot \left(1+O\left(\frac1m\right) \right)^q \\
	&\le \exp\left( \lambda c +  \frac{2\lambda^2q}{m} - \frac12 q \log\left(1+\frac{2\lambda}{m}\right)   \right) \cdot O(1) \\
	&\le \exp\left( \lambda c +  \frac{2\lambda^2q}{m} - \frac12 q \cdot \frac{4\lambda}{m}   \right) \cdot O(1) \\
	&= \exp\left( \lambda (c - q/2m) +  \frac{2\lambda^2q}{m}  \right) \cdot O(1) ,
	\end{align*}
	where we have used $\log(1+\alpha) \ge 2\alpha$ for $-\frac12 \le \alpha \le 0$.
	Letting $c =  - \log(0.9)$ and $\lambda = -\frac{cm}{10q} > -\frac m4$, and noting $m>\frac1c q$,
	we get
	\begin{equation} \label{eqn:chi-square-prod-lower-tail}
	\Pr\left[Z_q\cdots Z_1 < 0.9 m^q \right] \le O\left( e^{\frac{\lambda c}{2} + \frac{2\lambda^2q}{m}  } \right) = e^{-\Omega(m/q)}.
	\end{equation}
	Combining~\eqref{eqn:chi-square-prod-upper-tail}~and~\eqref{eqn:chi-square-prod-lower-tail}, we complete the proof.
\end{proof}

\section{Proof of Proposition~\ref{prop:initial-right-product-concentration}}
\label{app:proof-prop:initial-right-product-concentration}

\begin{proof}[Proof of Proposition~\ref{prop:initial-right-product-concentration}]
	
	Let $A = W_{i:1}(0) X$. Since $A \in \R^{m\times r}$ and $m>r$, we know 
	\begin{align*}
	\norm{A} &= \sup_{v\in \cS^{r-1}} \norm{Av},\\
	\sigma_{\min}(A) &= \inf_{v\in \cS^{r-1}} \norm{Av}.
	\end{align*}
	From Lemma~\ref{lem:matrix-prod-times-vector-concentration} we know that for any fixed $v\in \cS^{r-1}$, with probability at least $1 - e^{-\Omega(m/L)}$ we have $$\norm{Av} \in \left[ 0.9 m^{\frac{i}{2}} \norm{Xv}, 1.1 m^{\frac{i}{2}}\norm{Xv} \right].$$

	Now we take an $\eps$-net $\cN$ for $\cS^{r-1}$ with $\abs{\cN} \le (3/\eps)^{r}$,
	where $\eps\in(0, 1)$ is a parameter to be determined later.
	By a union bound, with probability at least $1 - \abs{\cN} e^{-\Omega(m/L)}$, for all $u\in \cN$ simultaneously we have $\norm{Au} \in \left[ 0.9 m^{\frac{i}{2}} \norm{Xu}, 1.1 m^{\frac{i}{2}}\norm{Xu} \right]$.
	Suppose this happens for every $u\in \cN$.
	Next, for any $v\in \cS^{r-1}$, there exists $u \in \cN$ such that $\norm{u-v} \le \eps$. Then we have
	\begin{align*}
	\norm{Av} &\le \norm{Au} + \norm{A(u-v)}\\
	&\le 1.1 m^{\frac{i}{2}} \norm{Xu} + \eps\norm{A} \\
	&\le 1.1 m^{\frac{i}{2}} \norm{X} \cdot (1+\eps)  + \eps\norm{A} .
	\end{align*}
	Taking supreme over $v\in \cS^{r-1}$, we obtain
	\begin{align*}
	\norm{A}  \le \frac{1.1 (1+\eps) m^{\frac{i}{2}}\norm{X}}{1-\eps}
	\le 1.2 m^{\frac{i}{2}}\norm{X},
	\end{align*}
	as long as $\eps \le 0.01$.
	
	For the lower bound, we have
	\begin{align*}
	\norm{Av} &\ge \norm{Au} - \norm{A(u-v)}\\
	&\ge 0.9 m^{\frac{i}{2}} \norm{Xu} - \eps\norm{A} \\
	&\ge 0.9 m^{\frac{i}{2}} \sigma_{\min}(X) \cdot (1-\eps) - \eps \cdot  1.2 m^{\frac{i}{2}}\norm{X}\\
	&\ge 0.8 m^{\frac{i}{2}} \sigma_{\min}(X),
	\end{align*}
	as long as $\eps \le \frac{0.01}{\norm{X}/\sigma_{\min}(X)} = \frac{0.01}{\sqrt{\kappa}}$.
	(Recall $\kappa = \frac{\lambda_{\max}(X^\top X)}{\lambda_r(X^\top X)} = \sigma_{\max}^2(X) / \sigma_{\min}^2(X)$.)
	Taking infimum over $v\in \cS^{r-1}$ we get $\sigma_{\min}(A) \ge 0.8 m^{\frac{i}{2}} \sigma_{\min}^2(X)$.
	
	Now we fix our choice $\eps = \frac{0.01}{\sqrt{\kappa}}$ which makes the above proof work.
	The success probability is at least $1 - \abs{\cN} e^{-\Omega(m/L)} = 1 - e^{-\Omega(m/L) + r \log(300\sqrt{\kappa})} = 1 - e^{-\Omega(m/L)}$
	since $m>CLr \log\kappa$.
\end{proof}

\section{Proof of Proposition~\ref{prop:init-loss-bound}} \label{app:proof-prop:init-loss-bound}

\begin{proof}[Proof of Proposition~\ref{prop:init-loss-bound}]
	From Lemma~\ref{lem:matrix-prod-times-vector-concentration} we know that for each input datapoint $x_p \in \R^d$ ($p\in[r]$), with probability at least  $1 - e^{-\Omega(m/L)}$ we have
	\begin{align*}
	\norm{W_{L-1:1}(0) x_p } \le 1.1 m^{\frac{L-1}{2}} \norm{x_p}.
	\end{align*}
	Taking a union bound, the above holds for all $p\in[r]$ with probability at least $1 - r e^{-\Omega(m/L)} = 1 - e^{-\Omega(m/L)}$ since $m>CL\log r$.
	
	Conditioned on $W_{L-1:1}(0)$, $Z_p =  \frac{ \norm{W_{L:1}(0) x_p}^2  }{ \norm{W_{L-1:1}(0) x_p}^2 }$ has distribution $\chi^2_{\dout}$.
	By the standard tail bound for $\chi^2$, with probability at least $1 - \delta'$ we have
	\begin{align*}
	Z_p \le \dout + 2\sqrt{\dout \log(1/\delta')} + 2\log(1/\delta').
	\end{align*}
	Letting $\delta' = \frac{\delta}{2r}$ and applying a union bound over $p\in[r]$, we have with probability at least $1-\delta/2$,
	\begin{align*}
	Z_p \le \dout + 2\sqrt{\dout \log(2r/\delta)} + 2\log(2r/\delta), \quad \forall p\in[r].
	\end{align*}
	
	Therefore we can bound the norm of the network prediction at initialization $U(0) = \frac{1}{\sqrt{m^{L-1}\dout}}W_{L:1}(0)X$ as
	\begin{align*}
	&\norm{U(0)}_F^2 \\
	=\,& \frac{1}{m^{L-1}\dout}\sum_{p=1}^{r} \norm{W_{L:1}(0)x_p}^2 \\
	=\,& \frac{1}{m^{L-1}\dout} \sum_{p=1}^{r} \norm{W_{L-1:1}(0)x_p}^2 Z_p \\
	\le\,& \frac{1.1^2}{\dout} \sum_{p=1}^{r} Z_p \norm{x_p}^2 \\
	\le\,& \frac{1.1^2}{\dout} \left( \dout + 2\sqrt{\dout \log(2r/\delta)} + 2\log(2r/\delta)  \right) \norm{X}_F^2 \\
	\le\,& O \left( \max\left\{ 1, \frac{\log(r/\delta)}{\dout}  \right\} \right) \norm{X}_F^2.
	\end{align*}
	Then the loss value $\ell(0)$ can be bounded as
	\begin{align*}
	\ell(0)  &= \frac12 \norm{U(0) - Y}_F^2
	\le \norm{U(0)}_F^2 + \norm{Y}_F^2 \\
	&\le \norm{U(0)}_F^2 + \norm{\Phi}^2\norm{X}_F^2 \\
	&\le O \left( \max\left\{ 1, \frac{\log(r/\delta)}{\dout}, \norm{\Phi}^2  \right\} \right) \norm{X}_F^2.\qedhere
	\end{align*}
\end{proof}

\section{Proof of Claim~\ref{claim:induction-3} (Continued from Section~\ref{subsec:induction-3})}
\label{app:proof-claim:induction-3 (cntd)}

\begin{proof}[Proof of Claim~\ref{claim:induction-3} (Continued)]
	
	Recall that what is missing from the proof in Section~\ref{subsec:induction-3} is to show the following:
	\[
	\frac{1}{\sqrt{m^{L-1}\dout}} \norm{E(t) X}_F \le  \frac16 \eta \lambda_{\min}(P_t) \norm{U(t)-Y}_F .
	\]
	
	Recall that $E(t)$ is the sum of all high-order terms in the product
	\begin{align*}
	W_{L:1}(t+1) = \prod_i \left( W_i(t) - \eta \frac{\partial \ell}{\partial W_i}(t)  \right).
	\end{align*}
	Same as~\eqref{eqn:gradient-norm-bound}, we have $\norm{\frac{\partial \ell}{\partial W_i} (t)}_F \le \frac{3 \sqrt{\ell(t)}\norm{X}}{\sqrt{\dout}}$ ($\forall i\in[L]$).
	Then we have
	\begin{align*}
	& \frac{1}{\sqrt{m^{L-1}\dout}} \norm{E(t) X}_F \\
	\le\,& \frac{1}{\sqrt{m^{L-1}\dout}} \sum_{s=2}^L \binom{L}{s} \left(\eta \cdot \frac{3 \sqrt{\ell(t)}\norm{X}}{\sqrt{\dout}} \right)^s   \cdot \left( O\left(\sqrt L\right) \right)^{s-1} m^{\frac{L-s}{2}} \norm{X} \\
	\le \,& \frac{1}{\sqrt{\dout}} \sum_{s=2}^L L^s \left(\frac{3 \eta \sqrt{\ell(t)}\norm{X}}{\sqrt{\dout}} \right)^s   \cdot  \left( O\left(\sqrt L\right) \right)^{s-1} m^{\frac{1-s}{2}} \norm{X} \\
	= \,& \frac{1}{\sqrt{\dout}} \sum_{s=2}^L  \left(  O\left( \frac{\eta L^{3/2}\sqrt{\ell(t)}\norm{X}}{\sqrt{\dout}}  \right)  \right)^s  L^{-1} m^{\frac{1-s}{2}} \norm{X} \\
	= \,& \frac{\norm{X}}{L\sqrt{\dout}} \sum_{s=2}^L  \left(  O\left( \frac{\eta L^{3/2} \sqrt{\ell(t)}\norm{X}}{\sqrt{m \dout}}  \right)  \right)^{s-1}       \cdot \frac{\eta L^{3/2} \sqrt{\ell(t)}\norm{X}}{\sqrt{\dout}}  \\
	= \,& \frac{\eta \sqrt{L\cdot \ell(t)}    \norm{X}^2  }{\dout} \sum_{s=1}^{L-1}  \left(  O\left( \frac{\eta L^{3/2} \sqrt{\ell(t)}\norm{X}}{\sqrt{m\dout}}  \right)  \right)^{s}   . 
	\end{align*}
	From $\eta \le \frac{\dout}{ 3L \norm{X}^2 }$ we have $\frac{\eta L^{3/2} \sqrt{\ell(t)}\norm{X}}{\sqrt{m\dout}} \le  \frac{  \sqrt{L\dout \cdot\ell(t)}}{3 \norm{X} \sqrt{m}} $.
	Since $m > C \cdot\frac{L \dout B}{\norm{X}^2} \ge C \cdot\frac{L \dout \ell(t)}{\norm{X}^2}$ (from~\eqref{eqn:m-bound-implication}) we have
	\begin{align*}
	& \frac{1}{\sqrt{m^{L-1}\dout}} \norm{E(t) X}_F \\
	\le\,& \frac{\eta \sqrt{L\cdot \ell(t)}    \norm{X}^2  }{\dout} \cdot \frac{  \sqrt{L\dout \cdot\ell(t)}}{ \norm{X}\sqrt{m}} \\
	=\,& \frac{\eta L\ell(t)   \norm{X}   }{\sqrt{m \dout}} .
	\end{align*}
	Now we would like the above bound to be at most $\frac16 \eta \lambda_{\min}(P_t) \norm{U(t)-Y}_F = \frac16 \eta \lambda_{\min}(P_t) \sqrt{2\ell(t)} $.
	Since $\lambda_{\min}(P_t) \ge \frac{3}{10} L \sigma_{\min}^2(X) / \dout$, it suffices to have
	\begin{align*}
	\frac{\eta L\ell(t)   \norm{X}   }{\sqrt{m \dout}} \le \frac16 \eta \cdot \frac{3 L \sigma_{\min}^2(X) \sqrt{2\ell(t)} }{ 10 \dout},
	\end{align*}
	which is true since $m> C \cdot \frac{\dout B \norm{X}^2}{\sigma_{\min}^4(X)} \ge C \cdot \frac{\dout \ell(t)\norm{X}^2}{\sigma_{\min}^4(X)}$  (from~\eqref{eqn:m-bound-implication}).
\end{proof}